\documentclass{article}
\usepackage{graphicx}
\usepackage{amsmath}
\usepackage{amsfonts}
\usepackage{mathtools}
\usepackage{booktabs}
\usepackage{multirow}
\usepackage{algorithm2e}
\usepackage[noend]{algorithmic}
\usepackage{subcaption}
\usepackage{tikz}
\usetikzlibrary{positioning,chains,fit,shapes,calc}
\usepackage{xcolor}
\usepackage{url}
\usepackage[switch]{lineno}
\usepackage{float}
\usepackage{fullpage}
\usepackage{hyperref}

\usepackage{natbib}

\newcommand{\qed}{\hfill \ensuremath{\Box}}

\newtheorem{lemmx}{Lemma}
\newtheorem{thmx}{Theorem}

\DeclareMathOperator*{\argmax}{arg\,max}
\newenvironment{proof}{{\bf Proof:}}{\qed}
\begin{document}

\title{COLOGNE: Coordinated Local Graph Neighborhood Sampling 
}


\author{Konstantin Kutzkov
\\ \footnotesize \href{mailto:kutzkov@gmail.com}{kutzkov@gmail.com} 
}


\date{}

\maketitle

\begin{abstract}
Representation learning for graphs enables the application of standard machine learning algorithms and data analysis tools to graph data. Replacing discrete unordered objects such as graph nodes by real-valued vectors is at the heart of many approaches to learning from graph data. Such vector representations, or embeddings, capture the discrete relationships in the original data by representing nodes as vectors in a high-dimensional space.
 
In most applications graphs model the relationship between real-life objects and often nodes contain valuable meta-information about the original objects. While being a powerful machine learning tool, embeddings are not able to preserve such node attributes. We address this shortcoming and consider the problem of learning discrete node embeddings such that the coordinates of the node vector representations are graph nodes. This opens the door to designing interpretable machine learning algorithms for graphs as all attributes originally present in the nodes are preserved.

We present a framework for coordinated local graph neighborhood sampling (COLOGNE) such that each node is represented by a fixed number of graph nodes, together with their attributes. Individual samples are coordinated and they preserve the similarity between node neighborhoods. We consider different notions of similarity for which we design scalable algorithms. We show theoretical results for all proposed algorithms. Experiments on benchmark graphs evaluate the quality of the designed embeddings and demonstrate how the proposed embeddings can be used in training interpretable machine learning algorithms for graph data.

%
\end{abstract}

\section{Introduction}
\label{intro}
Graphs are ubiquitous representation for structured data. They model naturally occurring relations between objects and, in a sense, generalize sequential data to more complex dependencies. Not surprising, many algorithms originally designed for learning from sequential data are generalized to learning from graph data. Learning vector representations of individual graph nodes, or {\em node embeddings}, is such an example where approaches to learning word representations from natural language text have inspired learning from graph data. Node embeddings have become an integral part of the graph learning toolbox, with applications ranging from link prediction~\citep{deepwalk,node2vec} to graph compression~\citep{compression}.

The first presented algorithm~\citep{deepwalk} for learning node embeddings works by generating random walks, starting from each node $u$ in the given graph, and then feeding the sequences of visited nodes into a word embedding learning algorithm such as word2vec~\citep{word2vec}. Later, these approaches were extended to a more general setting where different neighborhood definitions are supported and one can sample from these neighborhoods~\citep{line,node2vec,verse}. 
%
Algorithms based on random walks generate samples that are independent for different nodes, i.e., a random walk starting from a node $u$ is likely to be very different from a random walk starting at another node $v$, even if $u$ and $v$ have very similar neighborhoods, for some intuitive notion of similarity. By generating a large number of random walks the {\em sets} of sampled nodes for $u$ and $v$ will eventually reflect the similarity between $u$ and $v$ but individual samples are much less likely to preserve the similarity.  As an illustration, consider a large social network such as \texttt{twitter} where nodes represent users and edges who-follows-whom relationships. Consider two nodes corresponding to the famous football managers J\"urgen Klopp and Pep Guardiola. Maybe they don't have so many followers in common as they represent football rivals. Let $K$ and $G$ be Klopp's and Guardiola's 2-hop neighborhood local graphs, respectively. We might expect that the nodes in these ``friends of friends'' graphs, $V(K)$ and $V(G)$, will have a large overlap as they likely capture the majority of English football fans. However, if we sample nodes at random from $V(K)$ and $V(G)$, then most likely the sampled nodes will be different. But assume we were able to randomly permute the nodes in $V(K) \cup V(G)$ and then return the first node from $V(K)$ and $V(G)$, according to the total order defined by the permutation, as samples for Klopp and Guardiola. Then it would be much more likely to get the same node. In the case of random walks the probability for identical samples is $\frac{|V(G) \cap V(K)|}{|V(G)||V(K)|}$ and in the case of random permutations it is $\frac{|V(G) \cap V(K)|}{|V(G) \cup V(K)|}$, i.e., the samples already capture that Klopp and Guardiola are similar twitter users. And if we also consider the profiles of the sampled nodes, say by analyzing the activity of the corresponding users, then we might be able to infer that Klopp and Guardiola are Premier League managers. 

The main intuition behind  {\em coordinated} sampling is that each sample is an independent estimator of the similarity between nodes and thus sampled nodes themselves can be coordinates of the embedding vectors. This has two major advantages over continuous embeddings. First, we avoid the need to train a model that computes continuous embeddings. Thus, if the underlying sampling procedure is efficient, the approach can be highly scalable and simplify machine learning pipelines that work with node embeddings. Second, and probably more important, the samples are the original graph nodes. Usually graphs model real-life problems and graph nodes contain various kinds of additional information, be it personal data of users of a social network or the weather conditions at railway stations. By sampling, all this information is preserved which can lead to prediction models that are easier to interpret by a human expert.
%
%
\\\\
The main contributions of the paper can be summarized as follows:
\begin{itemize}
\item {\bf Coordinated local neighborhood sampling.} We formally define the problem of coordinated local sampling and present scalable algorithms with well understood theoretical properties.  The algorithms scale almost linearly with the graph size and yield samples that preserve the similarity between neighborhood nodes with respect to different objectives.

\item {\bf Interpretable embeddings.} The main motivation behind coordinated sampling is that it yields embeddings consisting of neighborhood nodes themselves. We show on real graphs how the information stored in nodes can be used to design interpretable machine learning models using embeddings consisting of coordinated samples.

\end{itemize}
\section{Organization of the paper} 
\label{sec:org}

In the next section we present notation and overview of techniques. In Section~\ref{sec:cologne} we first describe the overall structure of the proposed approach. We then present three algorithms for local neighborhood sampling according to different objectives. For each algorithm we also give theoretical results about the computational complexity of the approach and the properties of the returned samples. We discuss related work in Section~\ref{sec:previous}. An experimental evaluation on real graphs is presented in Section~\ref{sec:exp}. The paper is concluded in Section~\ref{sec:concl}.

\section{Notation and overview of techniques} \label{sec:overview}

We assume the input is a graph $G = (V, E)$ over $n=|V|$ nodes and $m=|E|$ edges. The distance $d(u,v)$ between nodes $u$ and $v$ is the minimum number of edges that need to be traversed in order to reach $u$ from $v$, i.e., the shortest path from $u$ to $v$. We consider undirected graphs, thus $d(u, v) = d(v, u)$. Also, we assume connected graphs, thus $d(u, v) < \infty$ for all $u, v \in V$. These assumptions are however only for the ease of presentation, all algorithms work for directed graphs and graphs with more than one connected component. The $k$-hop neighbors of node $u$ is the set of nodes $N_k(u) = \{v \in V: d(u, v) \le k\}$. The set of neighbors of node $u$ is denoted as $N(u)$. We call the subgraph induced by $N_k(u)$ the local $k$-hop neighborhood of $u$. The degree of node $u$ in graph $G=(V, E)$ is $deg(u)=|\{(u, v) \in E\}|$.

An $1\pm \varepsilon$-approximation of a quantity $q$ is another quantity $\tilde{q}$ such that $(1-\varepsilon)q \le \tilde{q} \le (1+\varepsilon)q$.
\paragraph{Coordinated sampling}
Given a universe of elements $U$, and a a set of sets $\{S_i \subseteq U\}$, the goal is draw a number of samples from $U$ such that each set $S_i$ is represented by a compact summary $sketch_{S_i}$ such that $\texttt{sim}(sketch_{S_i}, sketch_{S_j}) \approx \texttt{sim}(S_i, S_j)$, i.e., the summaries approximately preserve the similarity between the original sets, for different similarity measures. As an example, for a graph $G=(V, E)$ we can have $U = V$ and the sets $S_i$ be the neighbors of individual nodes. 
\paragraph{$L_p$ sampling for graph nodes}
The $p$-norm of vector $x \in \mathbb{R}^n$ is $\|x\|_p = (\sum_{i=1}^n x_i^p)^{1/p}$ for $p \in \mathbb{N}\cup \{0\}$. \footnote{The $0$-norm, counting the number of nonzero coordinates in $x$, is not a norm in the strict mathematical sense but the notation has become standard.} We call $L_p$ sampling a sampling procedure that returns each coordinate $x_i$ from vector $x$ with probability $\frac{|x_i|^p}{\|x\|_p^p}$.

\paragraph{Coordinated local graph neighborhood sampling} 
Let $\mathbf{f}^k_u$ be the $k$-hop frequency vectors of node $u$ such that $\mathbf{f}^k_u[z]$ the number of unique paths of length at most $k$ from $u$ to $z$. 
Let $s_u \in N_k(u)$ be the node returned by an algorithm $\mathcal{A}$ as a sample for node $u$. We say that $\mathcal{A}$ is a coordinated sampling algorithm with respect to a similarity measure $\texttt{sim}: V \times V \rightarrow [0,1]$ iff $$\Pr[s_u = s_v] \sim \texttt{sim}(\mathbf{f}^k_u, \mathbf{f}^k_v) \text{ for } u, v \in V$$

The objective of the present work is the design of scalable algorithms for coordinated $L_p$ sampling with rigorously understood properties. We can also phrase the problem in graph algebra terms. 
Let $A \in \{0,1\}^{n\times n}$ be the adjacency matrix of the graph. The objective is to implement coordinated $L_p$ sampling from each row of $ M_k = \sum_{i=0}^kA^i$  without explicitly generating $A^i$ where $A^0=I$. Note that it holds $M_k[u, z] = \mathbf{f}^k_u[z]$.

\paragraph{Sketch based coordinated sampling} Our algorithms will build upon sketching techniques for sampling from data streams. In a nutshell, sketching represents a massive input vector $x \in \mathbb{R}^n$ with a compact data structure $sketch_x \in \mathbb{R}^d$, $d \ll n$, that approximately preserves many properties of the original $x$. We want to sample from $sketch_x$, according to some distribution, such that the returned value is distributed as if we have sampled from $x$. We will apply sketching to summarize the local $k$-hop neighborhood frequency vector of each node and this will allow us to design efficient algorithms. The samples are coordinated by sharing the same random seed across neighborhoods. Denote by $s_u \in V$ the node selected as a sample for a node $u \in V$. Thus, for a node $x \in N_k(u) \cap N_k(v)$ the events $x=s_u$ and $x=s_v$ are not independent. If $s_u=x$, then it is more likely that we will also select $x$ as a sample for $v$. This is in contrast to random walks which are independent of each other. 

\section{COLOGNE sampling} \label{sec:cologne}

The general form of our approach is given in Figure~\ref{fig:main_alg}. We first initialize a sketch at each node $u$ with the node $u$ itself. Then for $k$ iterations for each node we collect the sketches from its neighbor nodes and aggregate them into a single sketch. At the end we sample from the sketch at each node. In this way we have aggregated the sketches from the $k$-hop neighborhood $N_k(u)$ for each node $u \in V$.

As a simple example assume that the sketch is a frequency vector $\mathbf{f}^k_u$ such that $\mathbf{f}^k_u[v]$ counts how many times a node $v$ occurs in $N_k(u)$. For each node $u$, $sketch_u$ is initialized by a sparse $\{0,1\}$-valued vector $\mathbf{f}_u$ such that $\mathbf{f}_u[v]=1$ iff $v=u$, i.e., there is exactly one nonzero coordinate. The aggregation is entrywise vector addition of all neighborhood frequency vectors. The next lemma formally shows that after $k$ iterations the value $sketch_u[v]$ is exactly the number of different paths from $u$ to $v$ of length {\em at most} $k$. 

\begin{lemmx} \label{lem:freq}
Let $G$ be a graph over $n$ nodes. Let $\mathbf{f}^{k}_u \in \mathbb{N}^n$ be the frequency vector collected at node $u$ after $k$ iterations of the  COLOGNE algorithm. Then $\mathbf{f}^{k}_u[v]$ is the number of unique walks from $u$ to $v$ of length at most~$k$.
\end{lemmx} 
\begin{proof}
We show that the entry $\mathbf{f}_u^{k}[v]$ corresponds to the number unique walk of length at most $k$ from $u$ to $v$ by induction on $k$. For $k=0$ the statement is trivial. For $k>1$ at each node $u$ we add up the frequency vectors collected at nodes $w \in N_k(u)$. By the induction assumption for $k-1$ the frequency vectors $\mathbf{f}_w^{k-1}[v]$ record the number of unique walks starting at $w$ and ending at $v$ of length  $t \le k-1$. Using that all neighbors of a node are distinct, by appending $u$ as a new starting node of each of these walks we create a new unique walk from $u$ to $v$ of length $t+1\le k$.   
\end{proof}

Sampling at random a node $v$ from $N_k(u)$ with probability $\mathbf{f}_u^k[v]/\|\mathbf{f}_u^k\|_1$ thus corresponds to a random walk starting at $u$ of length $k$.
Of course, a sketch that stores the entire frequency vectors $\mathbf{f}_u^k$ is not very useful. Even for smaller values of $k$, we are likely to end up with dense vectors at each node as most real-life networks have a small diameter which would lead to a total space of $O(n^2)$. Also, sampling an index from each frequency vector does not yield coordinated samples.

\begin{figure}
{\sc COLOGNE Sampling}
\begin{algorithmic}[1]
\REQUIRE Graph $G = (V, E)$
\FOR{each $u \in V$}
\STATE Initialize $sketch_u$ with node $u$ 
\ENDFOR
\FOR{$i=1$ to $k$}
\FOR{each $u \in V$}
\STATE Update $sketch_u$ by merging the sketches $sketch_v$ for $v \in N(u)$ into $sketch_u$ 
\ENDFOR
\ENDFOR
\FOR{$u \in V$}
\STATE {\bf Return} a sample from $sketch_u$
\ENDFOR
\end{algorithmic}
\caption{The overall structure of the coordinated local neighborhood sampling approach.}  \label{fig:main_alg} 
\end{figure}

\subsection{Lower bound}
If we have random access to each node's neighbors, then a random walk that samples from $N_k(u)$ can be implemented in time $O(k)$. But what about more general distributions? 
The following lemma gives a lower bound on the complexity of any algorithm that samples uniformly at random from a local neighborhood, i.e., for a node $u$ each node in $N_k(u)$ has equal chance of being sampled.
\begin{thmx} \label{thm:lb}
Any algorithm that queries less than $deg(u)/2$ edges in the local neighborhood of each $u \in V$, cannot obtain a $4/3$-approximation of uniform sampling.  
\end{thmx}
\begin{proof}
Consider a graph constructed as follows. Let $K_1$ and $K_2$ be two cliques, each on $n$ nodes. Choose at random $n/2$ nodes from both $K_1$ and $K_2$, call these the {\em marked} nodes and denote them as $M_1$ and $M_2$. For both $K_1$ and $K_2$ delete all edges between marked nodes and connect each node from $M_1$ to all nodes in $M_2$, thus creating a bipartite clique between marked nodes. Marked nodes have 2-hop neighborhood consisting of $2n-1$ nodes while unmarked nodes have a 2-hop neighborhood of  $3n/2-1$ nodes. Consider two nodes $u_K, u_M$ such that $u_K \in K_1\backslash M_1, u_M \in M_1$ and observe that by querying less than $n/2$ neighbors for $u_K$ and $u_M$ we can remain in $K_1$. By never reaching the nodes in $K_2$ we cannot distinguish between the neighborhood sizes of marked and unmarked nodes, and thus we cannot obtain an approximation ratio better than $2n/(3n/2)=4/3$. 
\end{proof}

The above result formalizes the intuition that in order to achieve uniform sampling we need to consider the entire neighborhood of each node, or $\Theta(m)$ edges for all nodes. Next we show that despite the lower bound we can still design scalable algorithms by computation sharing between  nodes. In fact, this is the key to coordinated sampling.

\subsection{Uniform ($L_0$) sampling}

We first present a simple algorithm for sampling uniformly at random from the local neighborhood of each node. The approach builds upon {\em min-wise independent permutations}~\citep{minwise}, a powerful technique for estimating the Jaccard similarity between sets. Assume we are given two sets $A \subseteq U$ and $B \subseteq U$, where $U$ is a universe of elements, for example all integers. We want to estimate the fraction $\frac{|A\cap B|}{|A\cup B|}$. Let $\pi: U \rightarrow U$ be a random permutation of the elements in $U$. With probability $\frac{|A\cap B|}{|A\cup B|}$ the smallest element in $A \cup B$ with respect to the total order defined by $\pi$ is contained in $A\cap B$ and thus the indicator variable denoting whether the smallest elements in $\pi(A)$ and $\pi(B)$ are identical yields an unbiased estimator of the Jaccard similarity  $J(A, B)$. The mean of $t = O(\frac{1}{\alpha \varepsilon^2})$ independent estimator is an $1\pm \varepsilon$-approximation of $J(A, B) \ge \alpha$ with probability more than 1/2. The success probability can be boosted to any $1-\delta$ by taking the median of $\log 1/\delta$ independent estimators. 

An algorithm for sampling uniformly from the $k$-hop neighborhood of node $u$ easily follows. We implement a random permutation on the $n$ nodes by generating a random number for each node $r: V \rightarrow \{0,1,\ldots,\ell-1\}$. For a sufficiently large $\ell$ with high probability $r$ is a bijective function and thus it implements a random permutation\footnote{For example, for $\ell = n^2/\delta$ with probability at least $1-\delta$ the function is bijective.}. For each node $u$, $sketch_u$ is  initialized with $(r(u), u)$, i.e., the sketch is just a single (random number, node) pair. The aggregation after each iteration is storing the pair with the smallest random number from $u$'s neighbors, $sketch_u = \min_{(r_v, v): v \in N(u)}sketch_v$. After $k$ iterations at each node $u$ we have stored the smallest number form the set $\{r(v): v \in N_k(u)\}$, i.e., we have sampled a node from $N_k(u)$ according to the permutation defined by the function $r$. Clearly, the samples for any two nodes $u$ and $w$ are coordinated as we work with the same permutation on the set $N_k(u) \cup N_k(w)$. The next theorem is a straightforward corollary from the main result on minwise-independent permutations~\citep{minwise}:

\begin{thmx}
For all nodes $u \in V$, we can sample $s_u \in N_k(u)$ with probability $1/|N_k(u)|$  in time $O(mk)$ and space $O(n)$. For any pair of nodes $u, v$ it holds $$\Pr[s_u=s_v] = \frac{|N_k(u) \cap N_k(v)|}{|N_k(u) \cup N_k(v)|}$$
\end{thmx} \label{thm:l0}
Note that for constant $k$ we match the lower bound from Theorem~\ref{thm:lb} and the space usage is $O(n)$ and not $O(m)$ because we need $k$ passes over the edges but they do not need to be stored in memory, but instead can be read from a secondary source or generated on the fly in arbitrary order.

In terms of linear algebra, the algorithm is an efficient implementation of the following approach: Let $A$ be the the adjacency matrix of $G$. Permute the columns of $M_k = \sum_{i=0}^kA^i$, and for each row in $M_k$ select the first nonzero coordinate. But we avoid the explicit generation of the powers $A^i$ which, as discussed in Section~\ref{sec:cont}, could be dense matrices even for small values of $i$. Thus, the algorithm implements coordinated $L_0$ sampling from each row of $M_k$.

\subsection{$L_p$ sampling} \label{sec:l0}

The solution for uniform sampling is simple and elegant but it does not fully consider the graph structure. We are only interested if there is a path between two nodes $u$ and $v$ but, unlike in random walks, not how many paths are there between $u$ and $v$. We might need a sampling probability  that is proportional to the probability that a node is accessed by a random walk but in the same time guarantees that sampling is coordinated, i.e., for a node $x \in N_k(u) \cap N_k(v)$ we want $$\Pr[s_u = x  \texttt{ and } s_v = x ] \gg \Pr[s_u = x] \Pr[s_v = x] $$

Let us first present an approach to $L_p$ sampling from data streams for $p \in (0, 2]$~\citep{lp_sampling,lp_sampling_survey}. Let $\mathcal{S}$ be a data stream of pairs $i, w_i$ where $i$ is the item and $w_i \in \mathbb{R}$ the weight update for item $i$, for $i \in [n]$. The objective is to return each item $i$ with probability roughly $|\mathbf{f}[i]|^p/\|\mathbf{f}\|^p_p$, where $\mathbf{f}[i] = \sum_{(i, w_i) \in \mathcal{S}}w_i$. (Note that items are supposed to occur multiple times in the stream.) The problem is trivial if we can afford to store the entire frequency vector $\mathbf{f}$ but for larger $n$ this can be prohibitively expensive. The idea is to reweight each item by scaling it by a random number $1/r_i^{1/p}$ for a uniformly sampled $r_i \in (0, 1]$. Let $z_i = \mathbf{f}[i]/r_i^{1/p}$ be the reweighted weight of item $i$. The crucial observation is that $$\Pr[z_i \ge \|\mathbf{f}\|_p] = \Pr[r_i \le {\mathbf{f}[i]^p}/{\|\mathbf{f}\|_p^p}] = \frac{\mathbf{f}[i]^p}{\|\mathbf{f}\|_p^p}$$ Thus, we need to detect a reweighted item whose weight exceeds $\|\mathbf{f}\|_p$. The solution in~\citep{lp_sampling,lp_sampling_survey} is to show that with constant  probability there exists a unique item whose weight exceeds $\|\mathbf{f}\|_p$ and the total weight of all other items is bounded. Thus, if we know the value of $\|\mathbf{f}\|_p$, a space-efficient solution is to keep a sketch data structure form which we can detect the heavy hitter that will be the sampled item. Estimating the norm of a frequency vector is a fundamental problem with known solutions~\citep{ams,count_sketch}, therefore the approach yields a space-efficient solution to $L_p$ sampling from data streams.

We will follow the above approach but there are several challenges we need to address when applying it to local neighborhood graph sampling.
The main difference is that we cannot afford to explicitly generate all entries in the frequency vector of the local neighborhood as this would result in time complexity of $O(\sum_{i=0}^kA^i)$, $A$ being the graph adjacency matrix. Also, estimating the norm of the frequency vectors has to be done in a new way because we cannot explicitly generate all updates to $\mathbf{f}_u^k$, the $k$-hop frequency vectors of node $u$. 

Our solution is based on the idea of {\em mergeable sketches}~\citep{mergeable}. Following the COLOGNE sampling template from Figure~\ref{fig:main_alg}, we  iteratively generate sketches at each node. The sketch collected in the $i$-th iteration will summarize the $i$-hop neighborhood of each node. In~\citep{lp_sampling} the authors use the Count-Sketch data structure~\citep{count_sketch}. This would fit perfectly into the COLOGNE framework as CountSketch is a linear projections of the data. Namely, an input vector $x \in \mathbb{R}^n$ is projected onto an $m$-dimensional subspace, $m \ll n$, as $Px$ for $P \in \mathbb{R}^{m\times n}$ where the projection matrix $P$ is implicitly defined using advanced hash functions. This means that for two vectors $x$ and $y$ we have $sketch(x + y) = sketch(x) + sketch(y)$ and we can iteratively sketch the $k$-hop frequency vector from the local neighborhood of each node. The result will be identical to first computing the frequency vector and then sketching it. Unfortunately, we cannot use Count-Sketch because we cannot efficiently retrieve the heavy hitter from the sketch. In the setting in~\citep{lp_sampling} a single vector is being updated in a streaming fashion. Thus after preprocessing the stream we can afford to query the sketch for each vector index $i \in [n]$ as this wouldn't increase the asymptotic complexity of the algorithm. But in our case this would mean we need to know all nodes in the local neighborhood of each node. And even if we knew those nodes, we would need $O(\sum_{i=0}^k nnz(A^i))$ queries in total which is likely to be of order $O(n^2)$.

Fortunately, the solution lies in applying another kind of summarization algorithms for frequent items mining, the so called {\em counter based algorithms}~\citep{frequent,spacesaving}. In this family of algorithms the sketch consists of an explicitly maintained list of frequent items candidates. For a sufficiently large yet compact sketch the heavy hitter is guaranteed to be detected. Next we obtain theoretical results for the approach.

\subsection*{$L_1$ sampling} \label{sec:l1}
In a nutshell, the algorithm works as follows. We adapt the approach from Lemma~\ref{lem:freq} but replace the frequency vector $\mathbf{f}^{k}_u$ at node $u$ with a {\em randomly weighted} vector  $\mathbf{w}^{k}_u$  such that $\mathbf{w}^k_u[v] = \mathbf{f}^k_u[v]/r_v$.  For each node $u$ we generate a random number $r_u \in (0, 1]$, initialize $\mathbf{w}^k_u[u] = 1/r_u$ and set $\mathbf{w}^k_u[v] = 0$ for all $v \neq u$.  We iterate over the neighbor nodes and update the vector $\mathbf{w}^k_u = \mathbf{w}^{k-1}_u \oplus \sum_{v \in N(u)} \mathbf{w}^{k-1}_v$ where $\oplus$ denotes entrywise vector addition.  We sample a node $v$ iff $\mathbf{w}_u[v] \ge t$ for some~$t$. Instead of explicitly maintaining the reweighted frequency $\mathbf{w}^k_u$, we keep a sketch from which we detect a heavy hitter that is then returned as a sample $s_u$. 

Let us provide some intuition before formally showing that the above approach yields coordinated sampling and preserves the connectivity properties of each node's neighborhood. First note that each node $v \in N_k(u)$ has a chance for being selected as $u$'s sample, i.e., $s_u=v$, because the numbers $r_v$ are chosen at random. Second, if there are many paths from $u$ to $v$ then $\Pr[s_u=v]$ is larger as $\mathbf{f}_u[v]$ is larger, in the same way as in random walks. And third, sampling is coordinated as for each $v \in  N_k(u) \cap  N_k(w)$, $v$ has a better chance to be the sample for $u$ and $w$, $s_u=s_w=v$, if $r_v$ is small which is in contrast to independent random walks.

What remains to consider is the computational complexity. As discussed, explicitly storing and updating the weighted vectors $\mathbf{w}^k_u$ for each node $u$ is not feasible because even for small $k$ this might lead to memory usage of $O(n^2)$. Instead, we will efficiently detect a node $x$ from a sketch of $u$'s $k$-hop weighted neighborhood frequency vector $\mathbf{w}_u^k$ for which it holds $\mathbf{w}^k_u[x] \ge t$. In this case the sketch will be realized using a frequent items mining algorithm such as~\citep{frequent} which detects heavy hitters in streams of weighted items without explicitly storing all items. 

We can phrase the algorithm again in terms of the adjacency matrix $A$ of $G$. Let $W=\sum_{i=0}^kA^i\cdot R$ where $R \in \mathbb{R}^{n\times n}$ is a diagonal matrix with diagonal entries randomly selected  from $(0, 1]$. Then from the $i$-th row we return as sample the index $j$ for which it holds $\argmax_{j} W_{i, j}$, i.e., the index with the maximum value in the row. Sampling in this way is coordinated because the $j$-th column of $\sum_{i=0}^kA^i$ is multiplied by the same random value $r_j$. 

\begin{thmx} Let $G$ be a graph over $n$ nodes and $m$ edges, and let $\mathbf{f}^k_u$ be the frequency vector of the $k$-hop neighborhood of node $u \in V$.
For all $u \in V$, we can sample a node $s_u \in N_k(u)$ with probability $\frac{\mathbf{f}^k_u[s_u]}{\|\mathbf{f}^k_u\|_1}$  in time $O(mk\log n)$  and space $O(n \log n)$. For each pair of nodes $u, v \in V$ $$\Pr[s_u=s_v] \sim \sum_{x \in V}\min(\frac{\mathbf{f}^k_u[x]}{\|\mathbf{f}^k_u\|_1}, \frac{\mathbf{f}^k_v[x]}{\|\mathbf{f}^k_v\|_1})$$
\end{thmx} \label{thm:l1}

\begin{proof}
We will show that with constant probability there exists exactly one node $x \in N_k(u)$ that satisfies the sampling condition for node $u$, i.e., $\mathbf{w}^k_u[x] \ge t$ for $t =  \|\mathbf{f}^k_{u}\|_1$ where $\mathbf{w}_u^k$ is the reweighted frequency vector $\mathbf{f}_u^k$. Denote this event by $\mathcal{E}_1$.
Consider a fixed node $x \in N_k(u)$.
First, node $x$ is reweighted by $r_x \sim U(0,1)$, therefore the probability that node $x$ satisfies the sampling condition is $$\Pr[\mathbf{f}^k_{u}[x]/r_x \ge t] = \frac{\mathbf{f}^k_{u}[x]}{\|\mathbf{f}^k_{u}\|_1}$$ 
For fixed $x \in N_k(u)$, let $\mathcal{E}^u_x$ be the event that $x$ is the only node that satisfies the sampling condition and for all other nodes $v \neq x$ it holds $\mathbf{w}^k_u[x] \le t/2$. We lower bound the probability for $\mathcal{E}^u_x$ as follows:
\begin{align}
\Pr[\mathcal{E}^u_x] =  \frac{\mathbf{f}^k_{u}[x]}{\|\mathbf{f}^k_{u}\|_1} \prod_{v \in N_k(u), v \neq x}(1-  \frac{2\mathbf{f}^k_{u}[v]}{t}) \ge \nonumber \\
\frac{\mathbf{f}^k_{u}[x]}{\|\mathbf{f}^k_{u}\|_1} \prod_{v \in N_k(u), v \neq x}\alpha(1-2/t)^{\mathbf{f}^k_{u}[v]} \ge  \texttt{(for $\alpha \in (0, 1]$)} \nonumber \\
\frac{\mathbf{f}^k_{u}[x]}{\|\mathbf{f}^k_{u}\|_1} \alpha (1-  2/t)^{\|\mathbf{f}^k_{u}\|_1} =  
\frac{\mathbf{f}^k_{u}[x]}{\|\mathbf{f}^k_{u}\|_1} e^{ -2\|\mathbf{f}^k_{u}\|_1/t} = \nonumber \\
e^{-2}\frac{\mathbf{f}^k_{u}[x]}{\|\mathbf{f}^k_{u}\|_1} \nonumber
\end{align}
%
The first inequality follows by observing that $(1-k/n)^n \rightarrow e^{-k}$ and $(1-1/n)^{kn}\rightarrow e^{-k}$ for large $n$, and since both $1-k/n$ and $1-1/n$ are in (0,1), for any fixed $n$ and $k < n$ there must exist a constant $\alpha \in (0,1)$ such that $1-k/n \ge \alpha(1-1/n)^k$.
The second inequality holds because $a^{\sum_{v \in V: v \neq x} w_v} > a^{\sum_{v \in V} w_v}$ for $a \in (0, 1)$ and $w_v > 0$.
\\
For $\mathcal{E}_1$ we observe that the events $\mathcal{E}^u_x$, $x \in N_k(u)$ are pairwise disjoint. $$\Pr[\mathcal{E}_1] = \sum_{x \in N_k(u)} \Pr[\mathcal{E}^u_x] \ge \sum_{x \in  N_k(u)} \frac{\mathbf{f}^k_{u}[x]}{e^2\|\mathbf{f}^k_{u}\|_1} = \Omega(1)$$
\\
Next we show  how to efficiently detect the unique element $x$ for which it holds $\mathbf{f}^k_{u}[x]/r_x \ge \|\mathbf{f}^k_{u}\|_1$. 
With probability at least $1-1/n$ it holds $r_x \in [1/n^2, 1]$ for all $x \in V$. We obtain for the expected value of $\mathbf{w}^k_u[x]=\mathbf{f}^k_{u}[x]/r_x$: 
\begin{align}
\mathbb{E}[\mathbf{f}^k_{u}[x]/r_x] = 
O(\mathbf{f}^k_{u}[x] \int_{1}^{n^2} \frac{1}{t} dt) = 
 O(\mathbf{f}^k_{u}[x]\log n) \nonumber
\end{align}
%
%
By linearity of expectation $\mathbb{E}[\|\mathbf{w}_u^k\|_1]={O}(\|\mathbf{f}^k_{u}\|\log n)$. Since the random numbers $r_x$ are independent by Hoeffdings's inequality we obtain that $\|\mathbf{w}_u^k\|_1 = {O}(\|\mathbf{f}^k_{u}\|\log n)$ almost surely. Note that $\|\mathbf{w}_u^k\|_1$ can be computed exactly at each node. Thus, we have shown that there exists a unique node $x$ with weight $\|\mathbf{f}^k_{u}\|_1$ and the total weight of the nodes in $N_k(u)$ is bounded by ${O}(\|\mathbf{f}^k_{u}\|_1\log n)$. Using a deterministic frequent items mining algorithm like {\sc Frequent}~\citep{frequent} we can detect this unique heavy hitter using space $O(\log n)$. Since for all other nodes $v \neq x$ it holds $\mathbf{w}^k_u[x] \le \|\mathbf{f}^k_{u}\|_1/2$, by the main result from~\citep{frequent} it follows that for a summary size of $> 2 \log n$ the heavy hitter will be the only node whose weight in the summary will be at least $\|\mathbf{f}^k_{u}\|_1/2$.  Note that the summaries generated by {\sc Frequent} are mergeable~\citep{mergeable}, and can be merged in time proportional to the number of elements stored in the summary. In each iteration we need to perform exactly $m$ such summary merges, thus each iteration over all nodes takes $O(m\log n)$ and needs space $O(n \log n)$  

To complete the proof consider the probability that $x \in  N_k(u) \cap N_k(v)$ is sampled for both nodes $u$ and $v$, i.e., $s_u=s_v=x$. As shown above, the probability $v \in N_k(u)$ are not sampled for all $v\neq x$ can be lower bounded by $\Omega(1)$.  Observe that $x=s_u$ if $r_x \le \mathbf{f}^k_{u}[x]/\|\mathbf{f}^k_{u}\|_1$, thus we have  
\begin{align}
\Pr[x = s_u \wedge x=s_v] = \Pr[\mathcal{E}_x^u \wedge \mathcal{E}_x^v]  = \nonumber \\
 \Theta(1) \Pr[r_x \le \mathbf{f}^k_{u}[x]/\|\mathbf{f}^k_{u}\|_1 \wedge r_x \le 
\mathbf{f}^k_{v}[x]/\|\mathbf{f}^k_{v}\|_1] \sim \nonumber \\
\min(\frac{\mathbf{f}^k_u[x]}{\|\mathbf{f}^k_u\|_1}, \frac{\mathbf{f}^k_v[x]}{\|\mathbf{f}^k_v\|_1}) \nonumber
\end{align}
The $\mathcal{E}^u_x$ events are disjoint for $x \in N_k(u)$. Also, nodes $x \notin  N_k(u) \cap N_k(v)$ contribute 0 to the similarity as they are not reachable by at least one of $u$ or $v$. Thus, summing over $x \in  N_k(u) \cap N_k(v)$ completes the proof. 
\end{proof}

\subsection*{$L_2$ sampling} \label{sec:l2}

The only obstacle that prevents us from applying the $L_1$ sampling algorithm to the $L_2$ case is that we cannot compute exactly the 2-norm of the weighted vector $\|\mathbf{w}_u^k\|_2$. We can compute $\|\mathbf{w}_u^k\|_1$ without knowing the values $\mathbf{f}_u^k[v]$ in advance, $v \in N_k(u)$, but we need to know the values $\mathbf{f}_u^k[v]^2$ in order to compute the 2-norm of the weighted frequency vector. However, we can efficiently {\em approximate} the 2-norm of a vector revealed in a streaming fashion, this is a fundamental algorithmic problem for which algorithms with optimal computational complexity have been designed~\citep{ams,count_sketch}. These algorithms can be adapted to the local graph neighborhood setting which yields the following result.  
 
\begin{thmx} Let $G$ be a graph over $n$ nodes and $m$ edges, and let $\mathbf{f}^k_u$ be the frequency vector of the $k$-hop neighborhood of node $u \in V$.
For all $u \in V$, we can sample a node $s_u \in N_k(u)$ with probability $(1\pm \varepsilon)\frac{\mathbf{f}^k_u[s_u]^2}{\|\mathbf{f}^k_u\|_2^2}$  in time $O(mk(1/\varepsilon^2 + \log n))$  and space $O(n (1/\varepsilon^2 + \log n))$, for a user-defined $\varepsilon \in (0, 1)$. For each pair of nodes $u, v \in V$ $$\Pr[s_u=s_v] \sim (1\pm \varepsilon)\sum_{x \in V}\min(\frac{\mathbf{f}^k_u[x]^2}{\|\mathbf{f}^k_u\|_2^2}, \frac{\mathbf{f}^k_v[x]^2}{\|\mathbf{f}^k_v\|_2^2})$$
\end{thmx}  \label{thm:l2}
\begin{proof}
The same proof as for $L_1$ sampling holds except for computing exactly the norm $\|\mathbf{w}_u^k\|_2$. Instead we will estimate it using {\sc CountSketch}~\citep{count_sketch}. This is a linear data structure that maintains an array with counters. The critical property of linear sketches is that for each node it holds $$sketch(\sum_{v \in N_k(u)} \mathbf{f}_v) = \sum_{v \in N_k(u)} sketch(\mathbf{f}_v)$$ Using a sketch with $O(1/\varepsilon^2)$ counters we obtain a $1\pm \varepsilon$ approximation of $\|\mathbf{w}_u^k\|_2$ for a user-defined $\varepsilon \in (0, 1)$. Thus, we can recover the heavy hitters in the local neighborhood of each node using the approach from the proof of Theorem 3. Observing that for any $\varepsilon \in (0, 1)$ there exists a constant $c$ such that $\frac{1}{1\pm \varepsilon} \in [1-c\cdot \varepsilon, 1+ c\cdot \varepsilon]$, we obtain the stated  bounds on the sampling probability. 
\end{proof}

\subsection{Discussion} 

Let us provide some intuition for the similarity measures approximated by the samples in Theorem~3 and Theorem~4. In $L_0$ sampling we treat all nodes in the local neighborhood equally, while in $L_1$ sampling the sampling probability is proportional to the probability that we reach a local neighbor by a random walk. $L_2$~sampling is biased towards high-degree local neighbors, i.e., if a node is reachable by many paths in the local neighborhood then it is even more likely to be sampled. 

The similarity function approximated by $L_0$ sampling is the Jaccard similarity between node sets in the local neighborhood. But the similarity for $L_1$ and $L_2$ sampling is less intuitive. Consider two vectors $x, y \in \mathbb{R}^n$. It holds $$0 \le \sum_{i=1}^n \min(\frac{x[i]^2}{\|x\|_2^2}, \frac{y[i]^2}{\|y|_2^2}) \le \sum_{i=1}^n \frac{x[i]^2}{2\|x\|_2^2} + \sum_{i=1}^n  \frac{y[i]^2}{2\|y|_2^2} = 1$$

In particular, if two nodes share no nodes in their $k$-hop neighborhoods the similarity is 0 and if they have identical frequency vectors the similarity is 1.

Also, for $a\ge 0, b\ge 0$ it holds $\min(a^2, b^2) \le ab$, thus we have that $$\sum_{i=1}^n \min(\frac{x[i]^2}{\|x\|_2^2}, \frac{y[i]^2}{\|y|_2^2}) \le \sum_{i=1}^n \frac{x[i]y[i]}{\|x\|_2 \|y\|_2} = \cos(x, y)$$

On the other hand, assume that $\|x\|_2 \le \|y\|_2$ and for all $i$ it holds $x[i] \le y[i] \le c x[i]$ for some $c>1$ and $\frac{x[i]}{\|x\|_2} \le \frac{y[i]}{\|y\|_2}$.  Then 
$$\sum_{i=1}^n  \min(\frac{x[i]^2}{\|x\|_2^2}, \frac{y[i]^2}{\|y|_2^2}) \ge \sum_{i=1}^n \frac{x[i]y[i]/c}{\|x\|_2 \|y\|_2} = \cos(x, y)/c$$

Thus, the measure in a sense approximates cosine similarity. And for $L_1$ sampling we obtain that the related measure is the so called sqrt-cosine similarity~\citep{sqrt_cos} $$\text{sqrt-cos}(x, y) = \sum_{i=1}^n \frac{x[i]y[i]}{\sqrt{\|x\|_1} \sqrt{\|y\|_1}} $$

\section{Extensions} \label{sec:ext}
\paragraph{Sampling from graph streams.}
For massive graphs, or graphs where the edges can only be implicitly generated and not persistently stored, the algorithm needs $k$ passes over the edges and the edges can be provided in arbitrary order. The memory usage is bounded by $O(n \log n)$, thus the algorithm works in the semi-streaming model of computation~\citep{semistream} where the space usage must be $O(n \text{ polylog }n)$.  
\paragraph{Exponential weight decay for nodes of larger distance.} If we consider nearby nodes more significant, then we could implement exponential weight decay. Namely, the sampling probability will decrease by $\lambda^{i}$, for some $\lambda < 1$, for nodes $v \in N_k(u)$ of distance $i$ from $s_u$. In linear algebraic terms we want to apply $L_p$ sampling to rows of $\sum_{i=0} \lambda^i A^i$, where are $A$ is the adjacency matrix. Adding exponential decay is straightforward for $L_1$ and $L_2$ sampling: in each next iteration we multiply the weight of the neighborhood nodes by $\lambda$. For $L_0$ sampling one can design heuristics that increase the weight of the nodes over each iteration in order to decrease its odds to be sampled. 
\paragraph{Node and edge weights.} The $L_1$ and $L_2$ sampling algorithms also naturally handle {\em nonnegative} node and edge weights. This is obvious for node weights as we simply reweight the original weight by the random weights. For edge weights, when collecting the neighborhood nodes, we will multiply node weights by the corresponding edge weight. However, the algorithms would fail if we allow negative weights. It is crucial that we summarize the frequency vectors using counter based frequent items mining algorithms and these algorithms do not work in the so called turnstile model~\citep{turnstile} where both positive and negative updates are allowed. 

Considering node or edge weights in the case of uniform sampling case is in a sense contradicting to the very purpose of the approach, namely to treat all local neighbors equally. However, in case one wants to disregard the graph structure but still take into account node weights, then algorithms for the estimation of {\em weighted Jaccard similarity}~\citep{weighted_minwise} can be adapted.  

\section{Previous work} \label{sec:previous}

\subsection{Coordinated sampling from graphs} Coordinated sampling~\citep{coord_sampling} is a widely used algorithmic technique for efficient similarity estimation. It has applications in areas ranging from genome-wide association studies~\citep{gene_lsh} to recommendation systems~\citep{recsys_lsh}.  It has been applied to summarizing massive graphs, examples include problems such as triangle counting in graph streams~\citep{local_triangles}, estimation of local and global clustering coefficients~\citep{lcc}, graph minor counting, etc. We refer to the survey~\citep{graph_stream_survey} for an overview of problems and algorithmic techniques for graph stream mining, many of which are based on ideas for coordinated sampling.  

\subsection{Continuous node embeddings}  \label{sec:cont}

\paragraph{Random walk based embeddings}
Representing words by real valued vectors, the so called word embeddings learnt from natural language corpora, has become ubiquitous and is a natural first step in many problems that require working with natural language. Not surprisingly, pioneering approaches in the area such as word2vec~\citep{word2vec} and GloVe~\citep{glove}  received a lot of attention and have become indispensable tools in natural language understanding. Given the wide range of applications involving graph data, word embedding learning has been extended to node embedding learning where the objective is for a given graph $G=(V, E)$ to learn a function $f: V\rightarrow \mathbb{R}^d$, i.e., a $d$-dimensional vector that represents a discrete object like a graph node by continuous features. Such node embeddings capture the structure of the underlying graph and enable to apply  machine learning algorithms to graph data in a straightforward way.

The first presented approach is based on random walks~\citep{deepwalk}. The main idea behind the algorithm is that for each graph node $u$, we learn how to predict $u$'s occurrence from its context. In natural language the context of each word is the set of surrounding words in a sequence, and for graph nodes the context is the set of local neighbors. Various algorithms have been proposed that allow some flexibility in selecting local neighbors according to different criteria, such as LINE~\citep{line}, PTE~\citep{pte}, node2vec~\citep{node2vec}, APP~\citep{app} and VERSE \citep{verse}.

\paragraph{Matrix factorization}

A branch of node embeddings algorithms work by factorization of (powers of) the adjacency matrix of the graph~\citep{hope,arope,grarep,comm_pres}. These algorithms have well-understood properties but can be inefficient as even if the adjacency matrix is usually sparse, its powers can be dense~\citep{small_world}, e.g., the average distance between any two users in the Facebook graph is only 4~\citep{4degrees}. The computational complexity is improved using advanced techniques from linear algebra. However, these algorithms do not yield interpretable embeddings.

Inspired by works that show that the word2vec model for learning word embeddings from natural text can be phrased as a matrix factorization problem, it has been shown that most of the random walk based approaches can be expressed as matrix factorization for implicitly defined matrices based on powers of the adjacency matrix of the original graph~\citep{matfac}.

\paragraph{Deep learning}
Finally, it is worth noting that node embeddings can be also learned using graph neural networks~\citep{graphsage}. The intermittent layer of a neural network whose input are individual nodes in fact trains embedding vectors per node. The main advantage of GNNs is that they are {\em inductive} and can be applied to previously unseen nodes, while the above discussed approaches are {\em transductive} and work only for a fixed graph. The disadvantage is that these are deep learning architectures that can be slow to train and might require careful hyperparameter tuning. 

\subsection{Coordinated local sampling for interpretable embeddings} \label{sec:nodesketch}
Approaches close to COLOGNE are NetHash~\citep{nethash} and NodeSketch~\citep{nodesketch}. NetHash uses similarity preserving hashing to generate embeddings for {\em attributed} graphs using minwise independent permutations. The algorithm however generates individual rooted trees of depth $k$ for each node and its complexity scales as $O(nt(m/n)^k)$ where $t$ is the maximum number of attributes per node. Our $L_0$ sampling algorithm is much more efficient as there is no need for the generation of separate trees for each node, thus we avoid the factor $(m/n)^k$. Also, COLOGNE does not need to assume, but can still handle, node attributes. 

NodeSketch is a heuristic for coordinated sampling from local neighborhoods. It builds upon algorithms for the estimation of the normalized min-max similarity between vectors $u, v \in \mathbb{R^+}^n$, $\|u\|_1=\|v\|_1=1$: $$\frac{\sum_{i=1}^n \min(u[i], v[i])}{\sum_{i=1}^n \max(u[i], v[i])}$$ The original NodeSketch approach does not consider interpretability but it is essentially a sampling based approach. It works by recursively sketching the neighborhood at each node until recursion depth $k$ is reached. It builds upon an algorithm for min-max similarity estimation~\citep{ioffe}. The approach is conceptually similar to our $L_p$ sampling algorithm but no theoretical guarantee for the quality of the returned samples can be provided. To understand the main difference, the main goal of COLOGNE is to detect a node that satisfies a predefined sampling condition. The sampling condition is the key to showing the stated sampling probability and the similarity function that pairs of sampled nodes approximate. NodeSketch assigns first random weights to nodes using the Ioffe algorithm, similarly to what we do in $L_1$ and $L_2$ sampling.  In the $i$-th recursive call of NodeSketch for each node $u$ we collect samples from $N(u)$ and add up the corresponding weights in case two nodes appear more than once as samples in $N(u)$. For node $u$ we select the node with the smallest weight from $N(u)$ and this is the crucial difference to COLOGNE. Working with a single node is subject to random effects that could make the behavior of the algorithm difficult to explain.  Consider the graph in Figure~\ref{fig:ns}. There are many paths from node $u$ to node $z$ in $N_2(u)$. In the first iteration of NodeSketch it is very possible that most of $u$'s immediate neighbors, the yellow nodes $v_1$ to $v_{12}$, will sample a blue node  as each $v$ is connected with several blue $w$ nodes, i.e., $s_{v_i} = w_j$. Thus, in the second iteration it is likely that $u$ ends up with a sample for $u$ different from $z$.  In contrast, by keeping a sketch for $L_1$ and $L_2$ sampling COLOGNE will provably preserve the information that node $z$ is reachable from $u$ by many different paths. And we can control the importance we assign to sampling easily reachable nodes by weighting the items with $r_i^{1/p}$, $r_i \in (0, 1]$ such in $L_2$ sampling we are more likely to have $s_u = z$.

\begin{figure}

\center{
\definecolor{myblue}{RGB}{80,80,160}
\definecolor{mygreen}{RGB}{80,160,80}
\definecolor{myorange}{RGB}{255,127,0}


\begin{tikzpicture}[thick,
	scale=0.75,
	transform shape,
  fsnode/.style={fill=myblue},
  ssnode/.style={fill=green},
  osnode/.style={fill=myorange},
  every fit/.style={ellipse,draw,inner sep=-2pt,text width=1.6cm},
  -,shorten >= 1pt,shorten <= 1pt
]
\node [fill=red,yshift=0cm,xshift=3cm,label=below:$u$](u)[] {};

\node [fill=yellow,yshift=1cm,xshift=-5.5cm,label=below:$v_1$](v1)[] {};
\node [fill=yellow,yshift=1cm,xshift=-4cm](v2)[] {};
\node [fill=yellow,yshift=1cm,xshift=-2.5cm](v3)[] {};
\node [fill=yellow,yshift=1cm,xshift=-1cm](v4)[] {};
\node [fill=yellow,yshift=1cm,xshift=0.5cm](v5)[] {};
\node [fill=yellow,yshift=1cm,xshift=2cm](v6)[] {};
\node [fill=yellow,yshift=1cm,xshift=3.5cm](v7)[] {};
\node [fill=yellow,yshift=1cm,xshift=5cm](v8)[] {};
\node [fill=yellow,yshift=1cm,xshift=6.5cm](v9)[] {};
\node [fill=yellow,yshift=1cm,xshift=8cm](v10)[] {};
\node [fill=yellow,yshift=1cm,xshift=9.5cm](v11)[] {};
\node [fill=yellow,yshift=1cm,xshift=11cm,label=below:$v_{12}$](v12)[] {};

\node [fill=red,yshift=2cm,xshift=3cm,label=below:$z$](z)[] {};

\draw (u) -- (v1);
\draw (u) -- (v2);
\draw (u) -- (v3);
\draw (u) -- (v4);
\draw (u) -- (v5);
\draw (u) -- (v6);
\draw (u) -- (v7);
\draw (u) -- (v8);
\draw (u) -- (v9);
\draw (u) -- (v10);
\draw (u) -- (v11);
\draw (u) -- (v12);

\draw (z) -- (v1);
\draw (z) -- (v2);
\draw (z) -- (v3);
\draw (z) -- (v4);
\draw (z) -- (v5);
\draw (z) -- (v6);
\draw (z) -- (v7);
\draw (z) -- (v8);
\draw (z) -- (v9);
\draw (z) -- (v10);
\draw (z) -- (v11);
\draw (z) -- (v12);

\node [fill=blue!50!white,yshift=3cm,xshift=-6cm,label=below:$w_1$](w1)[] {};
\node [fill=blue!50!white,yshift=3cm,xshift=-5cm](w2)[] {};
\node [fill=blue!50!white,yshift=3cm,xshift=-4cm](w3)[] {};
\node [fill=blue!50!white,yshift=3cm,xshift=-3cm](w4)[] {};
\node [fill=blue!50!white,yshift=3cm,xshift=-2cm](w5)[] {};
\node [fill=blue!50!white,yshift=3cm,xshift=-1cm](w6)[] {};
\node [fill=blue!50!white,yshift=3cm,xshift=0cm](wx)[] {};
\node [fill=blue!50!white,yshift=3cm,xshift=1cm](w7)[] {};
\node [fill=blue!50!white,yshift=3cm,xshift=2cm](w8)[] {};
\node [fill=blue!50!white,yshift=3cm,xshift=3cm](w9)[] {};
\node [fill=blue!50!white,yshift=3cm,xshift=4cm](w10)[] {};
\node [fill=blue!50!white,yshift=3cm,xshift=5cm](w11)[] {};
\node [fill=blue!50!white,yshift=3cm,xshift=6cm](w12)[] {};
\node [fill=blue!50!white,yshift=3cm,xshift=7cm](w13)[] {};
\node [fill=blue!50!white,yshift=3cm,xshift=8cm](w14)[] {};
\node [fill=blue!50!white,yshift=3cm,xshift=9cm](w15)[] {};
\node [fill=blue!50!white,yshift=3cm,xshift=10cm](w16)[] {};
\node [fill=blue!50!white,yshift=3cm,xshift=11cm](w17)[] {};
\node [fill=blue!50!white,yshift=3cm,xshift=12cm,label=below:$w_{19}$](w18)[] {};

\draw (z) -- (w3);
\draw (z) -- (w4);
\draw (z) -- (w5);
\draw (z) -- (w8);
\draw (z) -- (w10);

\draw (v1) -- (w3);
\draw (v2) -- (w4);
\draw (v5) -- (w5);
\draw (v7) -- (w15);
\draw (v7) -- (w1);
\draw (v5) -- (w2);
\draw (v9) -- (w7);
\draw (v3) -- (w2);
\draw (v2) -- (w3);
\draw (v4) -- (w5);
\draw (v6) -- (w6);
\draw (v5) -- (w2);
\draw (v5) -- (w8);
\draw (v6) -- (w9);
\draw (v6) -- (w1);
\draw (v9) -- (w13);
\draw (v7) -- (w12);
\draw (v7) -- (w7);
\draw (v3) -- (w7);
\draw (v4) -- (w8);
\draw (v11) -- (w10);
\draw (v7) -- (w11);
\draw (v8) -- (w11);
\draw (v12) -- (w11);
\draw (v12) -- (w9);
\draw (v10) -- (w12);
\draw (v11) -- (w9);
\draw (v7) -- (w9);
\draw (v11) -- (w12);
\draw (v3) -- (w1);
\draw (v10) -- (w12);
\draw (v4) -- (w5);
\draw (v5) -- (w4);
\draw (v7) -- (w18);
\draw (v10) -- (w18);
\draw (v8) -- (w17);
\draw (v11) -- (w17);
\draw (v12) -- (w16);
\draw (v7) -- (w16);
\draw (v11) -- (w15);
\draw (v7) -- (w14);
\draw (v10) -- (w14);
\draw (v10) -- (w13);
\draw (v3) -- (w6);
\draw (v3) -- (wx);
\draw (v1) -- (wx);
\draw (v5) -- (wx);

\end{tikzpicture}
}
\caption{NodeSketch~\citep{nodesketch}  might miss that there are many path from $u$ to $z$ of length at most 2. Best viewed in color.}
\label{fig:ns}
\end{figure}
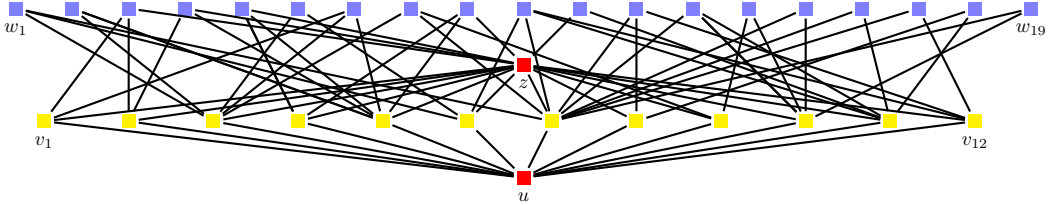

\section{Experiments} \label{sec:exp}

%
%
%
\paragraph{Datasets}
\begin{table}[!h]
\begin{center}
\begin{tabular}{| c || ccccccc |}

\hline
Dataset & nodes &  edges & labels & multilabel & features & weights & diameter\\ 
\hline
\hline
Cora & 2.7K & 5.4K & 7 & no & 1.4K & no & 19\\
\hline
Citeseer & 3.3K & 4.7K & 6 & no & 3.7K & no & 28 \\
\hline
Pubmed & 19.7K & 44.3K & 3 & no & 500 & yes & 18\\
\hline
Wikipedia & 4.8K & 92.5K & 40 & yes & 0 & no & 3\\
\hline
PPI & 3.9K & 38.7K & 50 & yes & 0 & no & 8\\
\hline
BlogCatalog & 10.3K & 333.9K & 40 & yes & 0 & no & 5\\ 
\hline
\end{tabular}
\caption{Information on datasets. The column {\em labels} denotes the total number of node labels, {\em multilabel} shows if a node can be assigned more than one class label, {\em features} is the total number of features that can be assigned to nodes, {\em weights} shows if node features are weighted. The last parameter is the graph diameter, i.e., the maximum distance between any two nodes.} \label{tab:datainfo}
\end{center}
\end{table}
We evaluated COLOGNE sampling against known approaches to local neighborhood sampling on six publicly available graph datasets, summarized in Table~\ref{tab:datainfo}. The first three datasets Cora, Citeseer, Pubmed~\citep{sen_et_al} are citation networks where nodes correspond to papers and edges to citations. Each node is assigned a unique class label, and nodes are described by a list of attributes, for example key words that appear in the article. In Pubmed the attributes for each node are key words with tf-idf scores, in Cora and Citeseer they are unweighted.

The next three datasets, Wikipedia~\citep{wikipedia}, PPI~\citep{ppi} and BlogCatalog~\citep{blogcatalog}, are so called multiclass multilabel problems, i.e., each node is assigned several class labels. The Wikipedia graph represents word co-occurrences and the labels are part-of-speech (POS) tags. PPI is a protein-protein interaction network where labels represent biological states. BlogCatalog is the graph of bloggers, edges represent relations between bloggers and labels the interest tags used by the bloggers.

\paragraph{Experimental setting}
The algorithms were implemented in Python 3 and experiments were performed on a Linux machine with a 3.9 GHz Intel CPU and 16 GB main memory\footnote{The implementation can be found in \url{https://github.com/kingkonk81/cologne}}. 
We generated samples from the $k$-hop neighborhood of each node for $k \in \{1,2,3,4\}$. We used following methods: 
\begin{enumerate}
\item Random walks (RW). For a node $u \in V$, we select at random a neighbor $v \in N(u)$, then sample one of $v$'s neighbors $w \in N(v)$, and so on. After $k$ iterations we return the last selected node. There are many random walk variations with different objectives but since no of them consider coordinated sampling, we only compare with the standard random walk approach. 
\item NodeSketch (NS). We select a node using the approach from~\citep{nodesketch}, see the original paper for details. 
\item Uniform sampling ($L_0$). The minwise sampling approach described in Section~\ref{sec:l0}.
\item $L_1$ sampling ($L_1$). We somewhat simplified the algorithm in Section~\ref{sec:l1} and return as a sample the heaviest element in the reweighted neighborhood frequency vector, as detected by the frequent items mining algorithm.  We used a sketch with 10 nodes. Note that in order to obtain the bounds in Theorem 3 we allowed a node to be sampled if certain constraints are satisfied but these techniques are only of theoretical interest such that we can mathematically analyze the algorithm.
\item $L_2$ sampling ($L_2$). We return as a sample the heaviest node in the reweighted frequency vector as described in Section~\ref{sec:l2}.  We used again a sketch with 10 nodes.
\end{enumerate}

%
\begin{table}[!h]
\begin{center}
\begin{tabular}{| c || c | ccccc |}
\hline
Dataset & $d$ & RW & NS & $L_0$ & $L_1$ & $L_2$\\ 
\hline
\hline
\multirow{3}{*}{Cora} & 10 & 0.8 & 3.1 & 2.8 & 7.2 & 7.4 \\

& 25 & 2.2 & 7.9 & 7.2 & 21.6 & 21.9\\

& 50 & 4.2 & 18.2  & 16.3 & 44.5 & 44.9\\
 \hline
\multirow{3}{*}{Citeseer} & 10 & 1.1 & 6.3 & 5.7 & 11.5 & 11.8 \\
 & 25 & 2.4 & 15.7 & 17.0 & 30.4 & 31.1\\
 & 50 & 6.9 & 36.6 & 33.9 & 61.3 & 62.3 \\
 \hline
 \multirow{3}{*}{Pubmed} & 10 & 6.8 & 67.1 & 61.5 & 104.4 & 104.9 \\
 & 25 & 19.1 & 161.2 & 157.1 & 295.6 & 296.5\\
 & 50 & 38.3 & 322.4 & 304.2 & 528.6 & 535.9\\
 \hline
 \multirow{3}{*}{PPI} & 10 & 1.6 & 1.9 & 1.2 & 14.6 & 16.1 \\
 & 25 & 3.9 & 4.5 & 3.1 & 41.6 & 42.8 \\
 & 50 & 9.1 & 11.8 & 7.0 & 84.0 & 83.8 \\
 \hline
 \multirow{3}{*}{Wikipedia} & 10 & 3.4 & 3.4 & 2.3 & 30.0 & 30.1 \\
 & 25 & 11.2 & 10.1 & 6.7 & 69.6 & 72.1\\
 & 50 & 17.0 & 19.7 & 12.9 & 148.2 &  157.5\\
 \hline
 \multirow{3}{*}{BlogCatalog} & 10 & 10.5 & 14.5 & 8.2 & 91.0 & 92.9 \\
 & 25 & 30.5 & 38.7 & 22.4 & 227.6 & 229.2\\
 & 50 & 58.3 & 74.7 & 50.5 & 458.8 & 464.6 \\
 \hline
\end{tabular}
\caption{Running time (in seconds) for feature generation from the 4-hop neighborhood for the five methods and different embeddings sizes, denoted as $d$.}
\label{tab:rtime}
\end{center}
\end{table}
\paragraph{Running time}
In Table~\ref{tab:rtime} we list the time needed to generate all node samples for varying number $d$ of samples per node.  We observe that random walks, NodeSketch and uniform ($L_0$) sampling are very efficient. This is because we keep a single sample per node during the $k$ iterations. $L_1$ and $L_2$ sampling are slower as we need to keep a summary of neighborhood nodes in order to detect a heavy hitter (the logarithmic factor in the running time in the theoretical analysis). Observe that the running time grows linearly with the embedding size, especially between 25 and 50 we have almost perfect doubling. Note that these are results obtained from a single-core implementation and since sample generation is easily parallelizable the approach can be significantly sped up by designing an advanced parallel architecture. 

\subsection{Link prediction}


\begin{figure}[ht!]
\centering
\begin{tabular}{cc}
  \includegraphics[width=60mm]{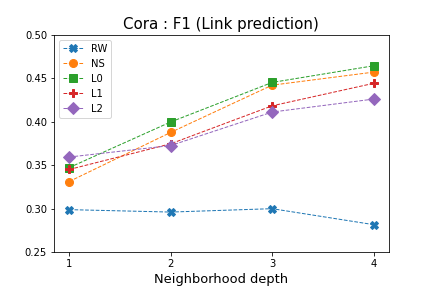} &   \includegraphics[width=60mm]{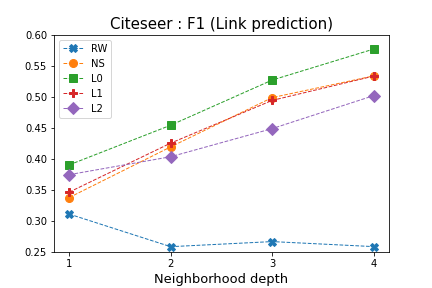} \\
 \includegraphics[width=60mm]{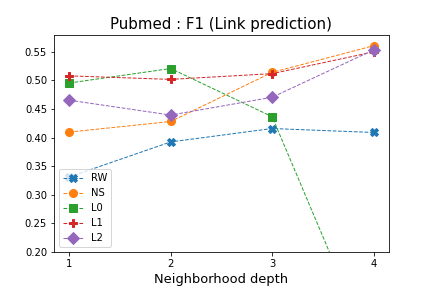} &   \includegraphics[width=60mm]{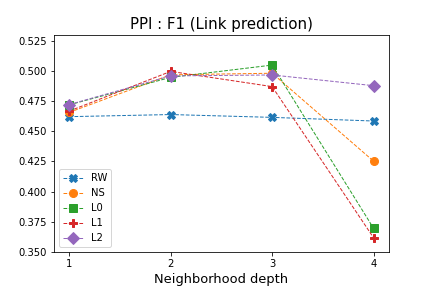} \\
 \includegraphics[width=60mm]{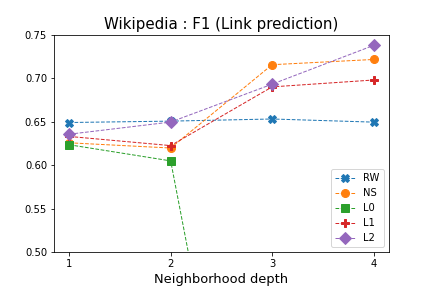} &   \includegraphics[width=60mm]{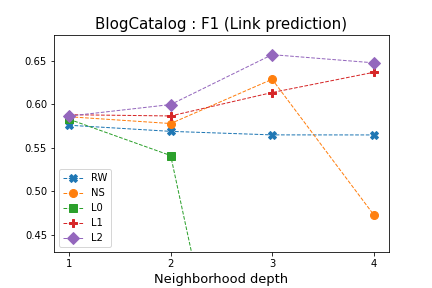}
\end{tabular}
\caption{$F_1$ scores for link prediction. Best viewed in color.} \label{fig:lp_f1}
\end{figure}
\begin{table*}
\begin{center}
\begin{tabular}{|c || cc || cc || cc |}
\multicolumn{1}{c}{} & \multicolumn{2}{c}{\large \texttt{Cora} } &  \multicolumn{2}{c}{\large \texttt{Citeseer}} &  \multicolumn{2}{c}{\large \texttt{Pubmed}}  \\ 
\hline
$k$ & COLOGNE & NodeSketch & COLOGNE & NodeSketch & COLOGNE & NodeSketch \\
\hline
\hline
1 &  0.359 ($L_2$) &  0.331 & 0.39 ($L_0$) &  0.338 &  0.508 ($L_1$) &  0.41\\
\hline
2 & 0.399 ($L_0$) &  0.387 & 0.454 ($L_0$) &  0.419 & 0.521 ($L_0$) &  0.428\\
\hline
3 & 0.445 ($L_0$) & 0.442 & 0.527 ($L_0$) &  0.499 & 0.512 ($L_1$) & 0.514\\
\hline
4 & {\bf 0.451} ($L_0$) & {\bf 0.458} & {\bf 0.578} ($L_0$) &  {\bf 0.532} & {\bf 0.553} ($L_2$) &  {\bf 0.558}\\
\hline
\end{tabular}
\end{center}
\vspace*{5mm}
%
\begin{center}
\begin{tabular}{|c || cc || cc || cc |}
\multicolumn{1}{c}{} & \multicolumn{2}{c}{\large \texttt{PPI} } &  \multicolumn{2}{c}{\large \texttt{Wikipedia}} &  \multicolumn{2}{c}{\large \texttt{BlogCatalog}}  \\ 
\hline
$k$ & COLOGNE & NodeSketch & COLOGNE & NodeSketch & COLOGNE & NodeSketch \\
\hline
\hline
1 & 0.472 ($L_2$) & 0.466 & 0.635 ($L_2$) &  0.626 &  0.588 ($L_1$) &  0.585\\
\hline
2 & 0.501 ($L_1$) &  0.497 &  0.649 ($L_2$)  &  0.62 &  0.6 ($L_2$) & 0.578\\
\hline
3 & {\bf 0.507} ($L_0$) & {\bf 0.498} &  0.693 ($L_2$) &  0.715 & {\bf 0.658} ($L_2$) & {\bf 0.627}\\
\hline
4 & 0.488 ($L_2$) & 0.425 &  {\bf 0.739} ($L_2$) &  {\bf 0.721} & 0.648 ($L_2$) &  0.473\\
\hline
\end{tabular}
\caption{Comparison of COLOGNE and NodeSketch for $F_1$. In bold font we give the best result for each dataset for COLOGNE and NodeSketch.}
\label{tab:F1_comp_other}
\end{center}
\end{table*}

The first set of experiments is for link prediction. We design a setting similar to the one in~\citep{node2vec}. We removed 20\% of the edges selected at random but such that the graph remains connected, and held them out as a test set. The rest of the edges together with four times more negative examples, i.e., pairs of nodes not connected by an edge, yield an imbalanced dataset with 20\% positive examples.  Each node $u$ is represented by a $d$-dimensional vector of discrete features $\mathbf{u} \in \mathbb{K}^d$ where $\mathbb{K}$ is the set of attributes describing nodes. We set $d=25$. For the citation networks the attributes are words describing the articles, and for PPI, Wikipedia and BlogCatalog we set the attribute to be the node itself, i.e. $\mathbb{K}=V$. For a pair of nodes $u$ and $v$ the input to a classification model is a $2d$-dimensional vector $z$ such that the entries in the $i$-coordinates in $u$ and $v$ are mapped to ``twin'' coordinates in $z$, i.e., the $i$-th coordinate in $u$ is mapped to the $2i$-th coordinate in $z$, and the $i$-th coordinate in $v$ is mapped to $(2i+1)$-th coordinate in $z$. For random walks we sampled an attribute at random, while for NodeSketch and COLOGNE we sampled an attribute in a coordinated way using the corresponding basic sampling algorithm where attributes share a random seed.  For example, in Pubmed nodes are described by a list of weighted features which are reweighted according to the algorithm used in COLOGNE or NodeSketch. 

We trained a Decision Tree classifier as it yields an explainable model. The splitting criterion is Gini index and the tree depth is unlimited.
The results in Figure~\ref{fig:lp_f1} show the $F_1$ scores for the six datasets for features collected at different neighborhood depths. Here are our observations:
\begin{itemize}
\item In general all coordinated sampling approaches yield comparable results. In Table~\ref{tab:F1_comp_other} we show that for some datasets COLOGNE is better than NodeSketch, so in downstream applications an optimal sampling approach might be considered a hyperparameter.
\item For networks of small diameter $L_0$ sampling yields very bad results for $k>2$. The reason is that all nodes end up with the same sample, namely the node with the smallest random seed since we disregard the graph structure. In contrast, $L_1$ and $L_2$ sampling are much less affected by the small diameter as we sample easily reachable neighborhood nodes. 
\item While NodeSketch yields overall good results its performance on BlogCatalog appears strange. We observe a sudden decrease of the predictive power of NodeSketch's embeddings when we increase the $k$-hop neighborhood from 3 to 4. Looking at the sampled values we observe that most nodes end up with identical samples. We can explain the behavior of $L_0$ sampling with the small diameter size but given the heuristic nature of NodeSketch it is challenging to make an educated guess why this happens for BlogCatalog but not for the Wikipedia graph which also has a small diameter.
\end{itemize}

\paragraph{Interpretability}

\begin{figure}
\centering
\begin{tabular}{cc}
  \includegraphics[width=60mm]{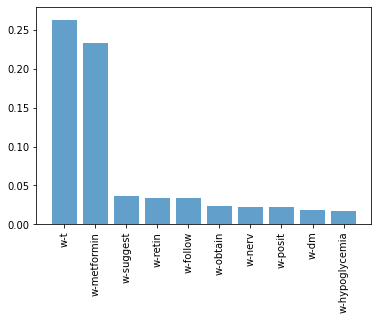} &   \includegraphics[width=60mm]{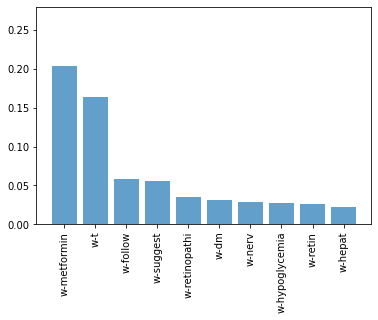} \\
\end{tabular}
\caption{Distribution of words in the most significant position according to the decision tree for $L_2$ sampling. On the left is the distribution for positive examples, i.e., nodes connected by an edge, in the right column is the distribution for negative examples.} \label{fig:distr}
\end{figure}

\begin{figure}
\scriptsize
\begin{verbatim}

Rules used to predict sample 1 with class 1:

16
decision node 0 : (X_test[1, 16] = ['w-t']) > 460.5)
17
decision node 9938 : (X_test[1, 17] = ['w-t']) > 460.5)
36
decision node 10958 : (X_test[1, 36] = ['w-young']) > 455.5)
37
decision node 11540 : (X_test[1, 37] = ['w-young']) > 496.5)
46
decision node 11938 : (X_test[1, 46] = ['w-bb']) <= 101.0)
35
decision node 11939 : (X_test[1, 35] = ['w-mous']) > 266.5)
47
decision node 11955 : (X_test[1, 47] = ['w-bb']) <= 85.0)
22
decision node 11956 : (X_test[1, 22] = ['w-inject']) > 235.0)
32
decision node 11978 : (X_test[1, 32] = ['w-experiment']) <= 341.5)
30
decision node 11979 : (X_test[1, 30] = ['w-observ']) > 62.0)
24
decision node 11985 : (X_test[1, 24] = ['w-cell']) <= 474.5)
11
decision node 11986 : (X_test[1, 11] = ['w-antigen']) <= 426.0)
...
prediction [1]


**********************************************

Rules used to predict sample 10000 with class 0:

16
decision node 0 : (X_test[10000, 16] = ['w-express']) <= 460.5)
17
decision node 1 : (X_test[10000, 17] = ['w-nerv']) <= 460.5)
30
decision node 2 : (X_test[10000, 30] = ['w-express']) > 62.0)
32
decision node 856 : (X_test[10000, 32] = ['w-beta']) > 56.0)
33
decision node 3178 : (X_test[10000, 33] = ['w-db']) > 67.5)
4
decision node 4358 : (X_test[10000, 4] = ['w-revers']) > 421.0)
5
decision node 6768 : (X_test[10000, 5] = ['w-revers']) > 421.0)
31
decision node 7462 : (X_test[10000, 31] = ['w-express']) > 62.0)
8
decision node 7530 : (X_test[10000, 8] = ['w-express']) <= 444.0)
9
decision node 7531 : (X_test[10000, 9] = ['w-progress']) <= 430.0)
44
decision node 7532 : (X_test[10000, 44] = ['w-express']) > 21.0)
45
decision node 7760 : (X_test[10000, 45] = ['w-femal']) > 37.5)
...
prediction [0]

\end{verbatim}
\caption{The decision rules for a positive and a negative example.}
\label{fig:rules}
\end{figure}

  \begin{figure}
\centering
\begin{tabular}{cc}
  \includegraphics[width=56mm]{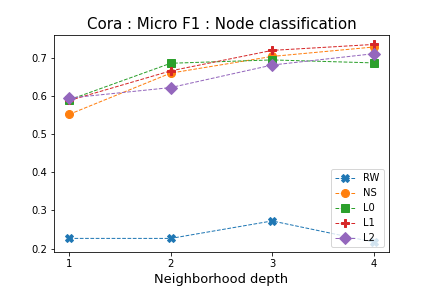} & \includegraphics[width=56mm]{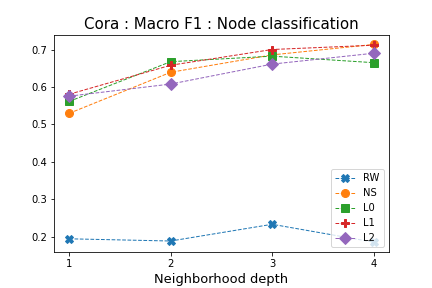} \\
   \includegraphics[width=56mm]{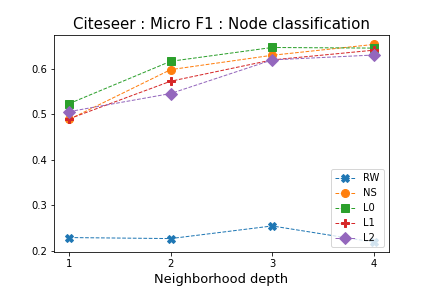} & \includegraphics[width=56mm]{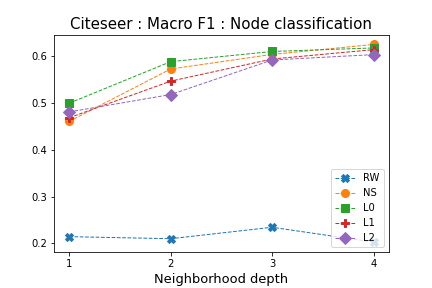}\\
 \includegraphics[width=56mm]{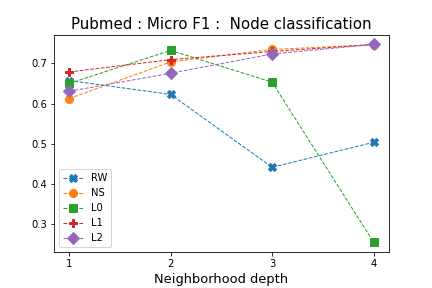} &  \includegraphics[width=56mm]{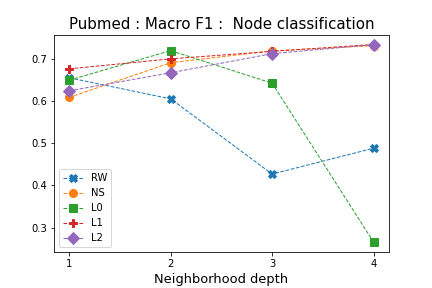} \\
 \includegraphics[width=56mm]{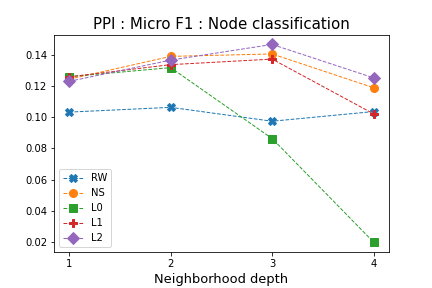} & \includegraphics[width=56mm]{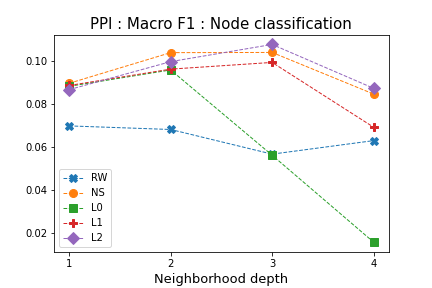}\\
 \includegraphics[width=56mm]{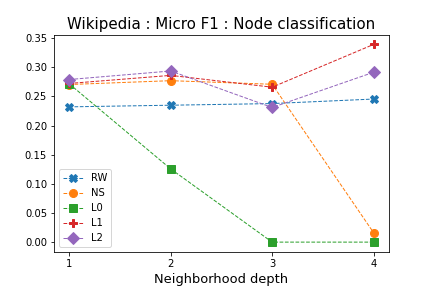} & \includegraphics[width=56mm]{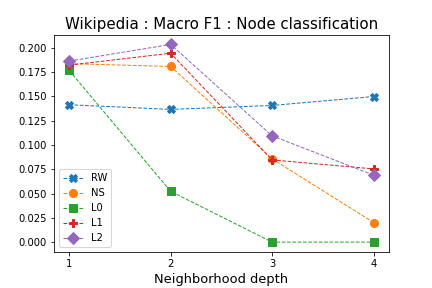} \\
  \includegraphics[width=56mm]{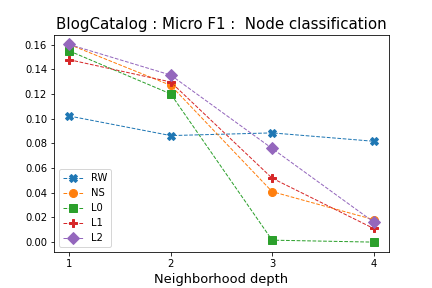} & \includegraphics[width=56mm]{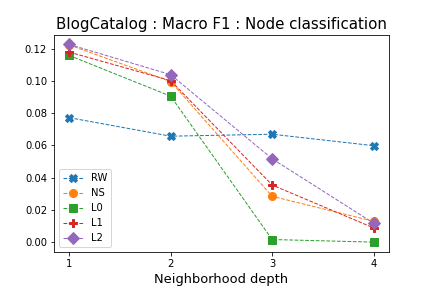}
\end{tabular}
\caption{Micro and Macro $F_1$ scores for node classification. Best viewed in color.}
\label{fig:clf_micro}
\end{figure}

A major advantage of coordinated sampling is that it yields interpretable node embeddings. First of all note that in all experiments the Decision tree models performed their first two splits on twin coordinates in the input vector, i.e., the indices at position $2i$ and $2i+1$ which represent the $i$-th samples for two nodes and are likely to contain identical samples for nodes with similar neighborhoods. In Figure~\ref{fig:distr} we plot the distribution of sampled words in the $L_2$ embeddings for the most significant position according to the decision tree model. On the left is the distribution for positive examples, and on the right is the distribution for negative examples. (Note that the distribution of the twin position is almost identical.) We observe clear differences between positive and negative examples which can help human experts to understand better the data.  

In Figure~\ref{fig:rules} we show the decision path for a positive and a negative example. (Note that the numeric values come from integer encoding of the word indices as computed by sklearn's LabelEncoder. Each leaf in the decision tree thus contains a small number of words.) We observe that in the positive example the two words in the top twin positions are identical while this is not the case for the negative examples. This is repeated at several deeper branches in the tree.

\subsection{Node classification}

\begin{figure}
\centering
\begin{tabular}{ccc}
  \includegraphics[width=45mm]{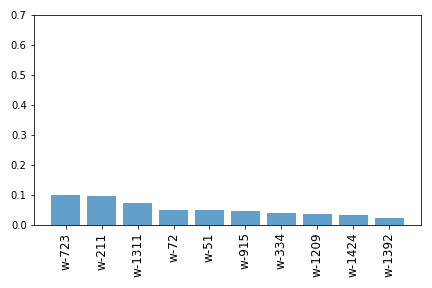} &  \includegraphics[width=45mm]{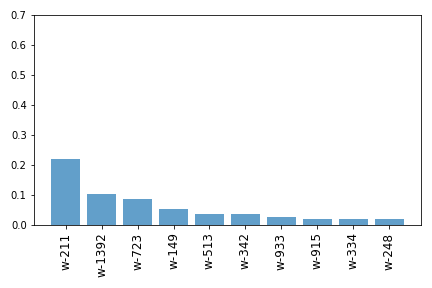} & \includegraphics[width=45mm]{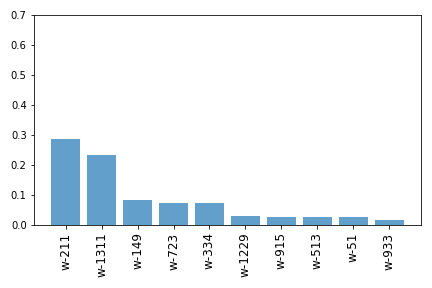}
\end{tabular}
\caption{Distribution of words for a coordinate of the embedding vectors for each of the three node classes of Pubmed.} \label{fig:pubmed_distr}
\end{figure}

We next consider the problem of node classification using the provided node labels as described in Table~\ref{tab:datainfo}. Using the same setting as for link prediction, each node is represented by a 50-dimensional vector such that the $i$-th coordinate is the attribute of a sampled node. We again train a decision tree classifier. We trained a one-vs-rest classifier on top of the decision tree model.   We split the data into 80\% for training and 20\% for testing and report average micro-$F_1$ and macro-$F_1$ scores for 10 independent runs of the algorithm in Figures~\ref{fig:clf_micro}. As expected, random walks never achieve better results than coordinated sampling. Analyzing the results we make similar observations as for link prediction. For graphs of small diameter $L_0$ sampling is only useful up to $k=2$. The results of COLOGNE and NodeSketch are in general comparable but there is again an example where NodeSketch's behavior is confusing. On the Wikipedia graph NodeSketch's  micro-$F_1$ score suddenly drops when adding one more iteration for sampling. The label distribution on Wikipedia is highly skewed. We can argue that $L_1$ and $L_2$ sampling generate embeddings that enable us to learn the most common labels and become better with increasing $k$ as they sample more representative nodes, hence the micro-$F_1$ scores improve but the macro-$F_1$ become worse. But we don't know how to explain the sudden decrease in both micro-$F_1$ and macro-$F_1$ values obtained when using NodeSketch's embeddings. 

The result are again interpretable for a human expert. In Figure~\ref{fig:pubmed_distr} we show how the different distribution of node attributes for the three node classes, for one coordinate of the embedding vector. Note that there is no such skew in the random walks embeddings.

\subsection{Comparison with continuous node embeddings}

We do not compare against continuous embeddings as these are not interpretable and the focus of the paper is to present and evaluate approaches to coordinated sampling. Such a comparison can be found in~\citep{nodesketch} where NodeSketch is compared against different approaches to the generation of continuous embeddings. 
However, we would like to mention that the results in~\citep{nodesketch} are obtained by training an SVM model with custom kernel where the kernel is a precomputed Gram matrix with the Hamming distance between node pairs. In our opinion it is questionable if such an approach is really feasible for dealing with large graphs. It is likely that one would need to consider $\Omega(n)$ nodes in order to train a good model. A Gram matrix with $O(n^2)$ entries can easily become a computational bottleneck. As evident from the theoretical analysis, the main effort in our design of COLOGNE is to guarantee that samples can be generated in time linear in the graph size.

We chose such a basic prediction model as decision trees are the textbook example for interpretable machine learning. It is worth noting that more accurate models can be obtained by using more advanced classification models. For example, using LightGBM~\citep{light_gbm}, a highy efficient framework for Gradient Tree Boosting. However, such results can be highly dependent on hyperparameter tuning and it becomes challenging to guarantee a fair comparison between the different sampling approaches. Moreover, the models become much more difficult to interpret as in order to achieve good results one needs hundreds and sometimes thousands of trees for an optimal model. 

\subsection{Unsupervised learning}
Finally, we would like to provide an example how coordinated sampling can be used in an unsupervised setting. 
 For each node in Pubmed we sampled 100 such words using $L_1$ sampling and random walks and, as there are 500 unique words (see Table~\ref{tab:datainfo}), converted them to sparse binary vectors of dimensionality 50000. These vectors were clustered using scikit-learn's K-means into 4 clusters using ``k-means++'' for initialization. We collected the 6 most frequently appearing words in the node descriptions in each cluster and plot the normalized word frequencies in Figure~\ref{fig:clusters}. Using $L_1$ sampling in 3 of the 4 clusters we have clearly dominating words while this is not the case for random walk based vectors. (Note that the $x$-axis is on a logarithmic scale.) 

\begin{figure*}
\centering
\begin{tabular}{c}
\includegraphics[scale=.4]{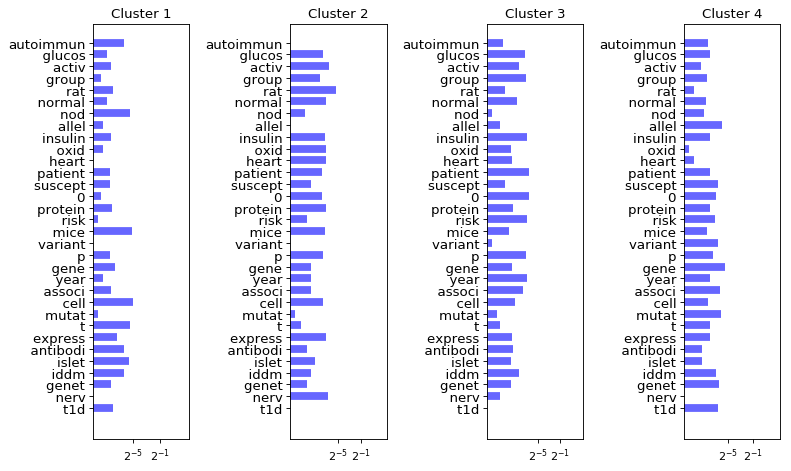} \\
\includegraphics[scale=.4]{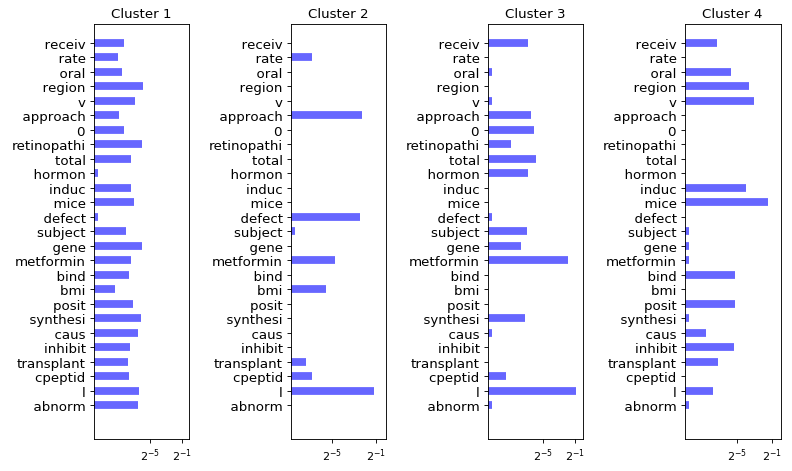}
\end{tabular}
\caption{Relative weight of the most significant words per cluster using vectors generated by random walk sampling (top) and $L_1$ sampling (bottom) on a logarithmic scale.}
\label{fig:clusters}
\end{figure*}


\section{Conclusions and future work} \label{sec:concl}

The paper lays the theoretical foundations for coordinated local graph sampling and argues that the approach can find practical applications by presenting experiments on real-life graphs. A deeper look into the graph structure can yield more insights into graph representation learning. We made certain observations about the graph diameter and the performance of the different sampling strategies, yet more advanced concepts from graph theory will likely lead to a better understanding of the performance of the algorithms. 

We would like to pose two open questions:
\begin{enumerate}
\item Can we design {\em interpretable} node embeddings for known similarity measures between the local neighborhood frequency vectors? For example, we can design a sketching algorithm to approximate the cosine similarity between $k$-hop frequency vectors of node pairs. Such an approach would build on the techniques we used in the proofs of Theorem~3 and Theorem~4. But sketches will consist of counters, not of sampled nodes, so they won't be interpretable.
\item Can we learn structural roles~\citep{struct_roles} by assigning appropriate node attributes? In particular, can we combine the sampling procedure with approaches for graph labeling such that we compute more informative node representations?
\end{enumerate}



%
%

\bibliographystyle{spbasic}      
\bibliography{cologne.bib}   


\end{document}